\def\year{2015}
\documentclass[letterpaper]{article}
\usepackage{aaai}
\usepackage{times}
\usepackage{helvet}
\usepackage{courier}
\usepackage{algorithm}
\usepackage{graphicx}
\usepackage{algorithmic}
\usepackage{subcaption}
\usepackage{amsmath}
\usepackage{amssymb}
\usepackage{amsthm}
\begin{document}
%
\title{Reinforcement Learning with Parameterized Actions}
\author{Warwick Masson \and Pravesh Ranchod\\
School of Computer Science and Applied Mathematics\\
University of Witwatersrand\\
Johannesburg, South Africa\\
warwick.masson@students.wits.ac.za\\
pravesh.ranchod@wits.ac.za
\And
George Konidaris\\
Department of Computer Science\\
Duke University\\
Durham, North Carolina 27708\\
gdk@cs.duke.edu
}
\maketitle
\begin{abstract}
We introduce a model-free algorithm for learning in Markov decision processes
with parameterized actions---discrete actions with continuous parameters.
At each step the agent must select both which action to use and which parameters to use with that action.
We introduce the Q-PAMDP algorithm for learning in these domains,
show that it converges to a local optimum, and compare it to direct policy search in the goal-scoring and Platform domains.
\end{abstract}

\section{Introduction}

Reinforcement learning agents typically have either a discrete or a continuous action space \cite{sutton1998introduction}.
With a discrete action space, the agent decides which distinct action to perform from a finite action set.
With a continuous action space, actions are expressed as a single real-valued vector.
If we use a continuous action space, we lose the ability to consider differences in kind: all actions must be expressible as a single vector.
If we use only discrete actions, we lose the ability to finely tune action selection based on the current state.

A parameterized action is a discrete action parameterized by a real-valued vector.
Modeling actions this way introduces structure into the action space by treating different kinds of
continuous actions as distinct.
At each step an agent must choose both which action to use and what parameters to execute it with.
For example, consider a soccer playing robot which can kick, pass, or run.
We can associate a continuous parameter vector to each of these actions: we can kick the ball to a given target position with a given force,
pass to a specific position, and run with a given velocity.
Each of these actions is parameterized in its own way.
Parameterized action Markov decision processes (PAMDPs)
model situations where we have distinct actions that require parameters to adjust the action to different situations,
or where there are multiple mutually incompatible continuous actions.

We focus on how to learn an action-selection policy given pre-defined parameterized actions.
We introduce the Q-PAMDP algorithm, which alternates learning 
action-selection and parameter-selection policies and compare it to a direct policy search method.
We show that with appropriate update rules Q-PAMDP converges to a local optimum.
These methods are compared empirically in the goal and Platform domains.
We found that Q-PAMDP out-performed direct policy search and fixed parameter SARSA.

\section{Background}

A Markov decision process (MDP) is a tuple $\langle S,A,P,R, \gamma \rangle$,
where $S$ is a set of states, $A$ is a set of actions, $P(s, a, s^\prime)$
is the probability of transitioning to state $s^\prime$ from state $s$ after taking action $a$,
$R(s,a, r)$ is the probability of receiving reward $r$ for taking action $a$ in state $s$,
and $\gamma$ is a discount factor \cite{sutton1998introduction}.
We wish to find a policy, $\pi(a|s)$, which selects an action for each state so as to maximize the expected sum of discounted rewards (the return).

The value function $V^\pi(s)$ is defined as the expected discounted return achieved by policy $\pi$ starting at state $s$
\[V^\pi(s) = \mathbb{E}_\pi \left[ \sum_{t=0}^\infty \gamma^t r_t \right].\]
Similarly, the action-value function is given by
\[Q^\pi(s,a) = \mathbb{E}_\pi \left[ r_0 + \gamma V^\pi(s^\prime) \right],\]
as the expected return obtained by taking action $a$ in state $s$, and then following policy $\pi$ thereafter.
While using the value function in control requires a model, we would prefer to do so without needing such a model.
We can approach this problem by learning $Q$,
which allows us to directly select the action which maximizes $Q^\pi(s,a)$.
We can learn $Q$ for an optimal policy using a method such as Q-learning \cite{watkins1992q}.
In domains with a continuous state space, we can represent $Q(s,a)$ using parametric function approximation with a set of parameters $\omega$
and learn this with algorithms such as gradient descent SARSA($\lambda$) \cite{sutton1998introduction}.

For problems with a continuous action space ($A \subseteq \mathbb{R}^m$), selecting the
optimal action with respect to $Q(s,a)$ is
non-trivial, as it requires finding a global maximum for a function in a continuous space.
We can avoid this problem using a policy search algorithm, where a class of policies parameterized by a set of parameters $\theta$ is given,
which transforms the problem into one of direct optimization over $\theta$ for an objective function $J(\theta)$.
Several policy search approaches exist, including policy gradient methods,
entropy-based approaches, path integral approaches, and sample-based approaches
\cite{deisenroth2013survey}.

\subsection{Parameterized Tasks}

A parameterized task is a problem defined by a task parameter vector $\tau$ given at the beginning of each episode.
These parameters are fixed throughout an episode, and the goal is to learn a task dependent policy.
Kober \textit{et al.} \shortcite{kober2012reinforcement} developed algorithms to adjust motor primitives to different task parameters.
They apply this to learn table-tennis and darts with different starting positions and targets.
Da Silva \textit{et al.} \shortcite{silva2012learning} introduced the idea of a parameterized skill as a task dependent parameterized policy.
They sample a set of tasks, learn their associated parameters,
and determine a mapping from task to policy parameters.
Deisenroth \textit{et al.} \shortcite{deisenroth2014multi} applied a model-based method to learn a task dependent parameterized policy.
This is used to learn task dependent policies for ball-hitting task, and for solving a block manipulation problem.
Parameterized tasks can be used as parameterized actions.
For example, if we learn a parameterized task for kicking a ball to position $\tau$,
this could be used as a parameterized action kick-to($\tau$).

\section{Parameterized Action MDPs}

\begin{figure}[t]
\centering
    \begin{subfigure}[t]{0.225\textwidth}
        \includegraphics[width=\textwidth]{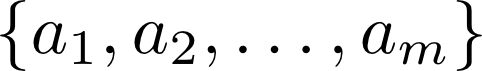}
        \caption{The discrete action space consists of a finite set of distinct actions.}
    \end{subfigure}
    ~
    ~
    ~
    \begin{subfigure}[t]{0.125\textwidth}
        \includegraphics[width=\textwidth]{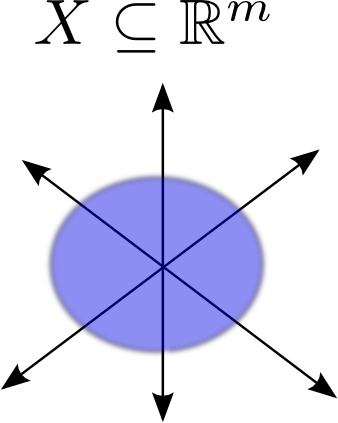}
        \caption{The continuous action space is a single continuous real-valued space.}
    \end{subfigure}
    \begin{subfigure}[t]{0.45\textwidth}
        \includegraphics[width=\textwidth]{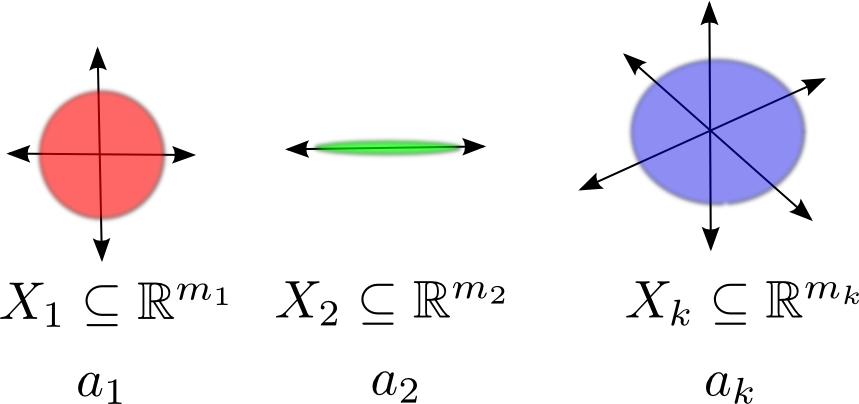}
        \caption{The parameterized action space has multiple discrete actions, each of which has a continuous parameter space.}
    \end{subfigure}
\caption{Three types of action spaces: discrete, continuous, and parameterized.}
\label{fig:action_spaces}
\end{figure}

We consider MDPs where the state space is continuous ($S \subseteq \mathbb{R}^n$) and the actions are parameterized:
there is a finite set of discrete actions $A_d = \{a_1,a_2,\ldots,a_k\}$,
and each $a \in A_d$ has a set of continuous parameters $X_a \subseteq \mathbb{R}^{m_a}$.
An action is a tuple $(a, x)$ where $a$ is a discrete action and $x$ are the parameters for that action.
The action space is then given by
\[A = \bigcup_{a \in A_d} \{(a,x)~|~x \in X_a\},\]
which is the union of each discrete action with all possible parameters for that action.
We refer to such MDPs as parameterized action MDPs (PAMDPs).
Figure \ref{fig:action_spaces} depicts the different action spaces.

We apply a two-tiered approach for action selection:
first selecting the parameterized action, then selecting the parameters for that action.
The discrete-action policy is denoted $\pi^d(a|s)$.
To select the parameters for the action, we define the action-parameter policy for each action $a$ as $\pi^a(x|s)$.
The policy is then given by
\[\pi(a,x|s) = \pi^d(a|s)\pi^a(x|s).\]
In other words, to select a complete action $(a,x)$,
we sample a discrete action $a$ from $\pi^d(a|s)$ and then sample a parameter $x$ from $\pi^a(x|s)$.
The action policy is defined by the parameters $\omega$ and is denoted by $\pi_\omega^d(a|s)$.
The action-parameter policy for action $a$ is determined by a set of parameters $\theta_a$, and is denoted $\pi^a_\theta(x|s)$.
The set of these parameters is given by $\theta = [\theta_{a_1},\ldots,\theta_{a_k}]$.

The first approach we consider is direct policy search.
We use a direct policy search method to optimize the objective function.
\[J(\theta, \omega) = \mathbb{E}_{s_0 \sim D} [ V^{\pi_\Theta}(s_0)].\]
with respect to $(\theta, \omega)$, where $s_0$ is a state sampled according to the state distribution $D$.
$J$ is the expected return for a given policy starting at an initial state.

Our second approach is to alternate updating the parameter-policy and learning an action-value function for the discrete actions.
For any PAMDP $M = \langle S, A, P, R, \gamma \rangle$ with a fixed parameter-policy $\pi^a_\theta$,
there exists a corresponding discrete action MDP, $M_\theta = \langle S, A_d, P_\theta, R_\theta, \gamma \rangle$, where
\begin{align*}
P_\theta(s^\prime|s,a) &= \int\limits_{x \in X_a} \pi_\theta^a(x|s) P(s^\prime|s,a,x)dx,\\
R_\theta(r|s,a) &= \int\limits_{x \in X_a} \pi_\theta^a(x|s) R(r|s,a,x)dx.
\end{align*}

We represent the action-value function for $M_\theta$ using function approximation with parameters $\omega$.
For $M_\theta$, there exists an optimal set of representation weights
$\omega_{\theta}^*$ which maximizes $J(\theta, \omega)$ with respect to $\omega$.
Let
\[W(\theta) = \arg\max\limits_{\omega} J(\theta, \omega) = \omega_{\theta}^*.\]
We can learn $W(\theta)$ for a fixed $\theta$ using a Q-learning algorithm.
Finally, we define for fixed $\omega$,
\begin{align*}
J_\omega(\theta) &= J(\theta, \omega), \\
H(\theta) &= J(\theta, W(\theta)).
\end{align*}
$H(\theta)$ is the performance of the best discrete policy for a fixed $\theta$.

\begin{algorithm}[tb!]
    \caption{Q-PAMDP($k$)}
    \label{alg:Q-PAMDP}
\begin{algorithmic}
    \STATE {\bfseries Input:}
    \STATE Initial parameters $\theta_0, \omega_0$
    \STATE Parameter update method P-UPDATE
    \STATE Q-learning algorithm Q-LEARN
    \STATE {\bfseries Algorithm:}
    \STATE $\omega \gets \text{Q-LEARN}^{(\infty)}(M_{\theta}, \omega_0)$
    \REPEAT
        \STATE $\theta \gets \text{P-UPDATE}^{(k)}(J_\omega, \theta)$
        \STATE $\omega \gets \text{Q-LEARN}^{(\infty)}(M_\theta, \omega)$
    \UNTIL $\theta$ converges
\end{algorithmic}
\end{algorithm}

Algorithm \ref{alg:Q-PAMDP} describes a method for alternating updating $\theta$ and $\omega$.
The algorithm uses two input methods: P-UPDATE and Q-LEARN and a positive integer parameter $k$, which determines the number of updates to $\theta$ for each iteration.
P-UPDATE($f, \theta)$ should be a policy search method that optimizes $\theta$ with respect to objective function $f$.
Q-LEARN can be any algorithm for Q-learning with function approximation.
We consider two main cases of the Q-PAMDP algorithm: Q-PAMDP($1$) and Q-PAMDP($\infty$).

Q-PAMDP(1) performs a single update of $\theta$ and then relearns $\omega$ to convergence.
If at each step we only update $\theta$ once, and then update $\omega$ until convergence, we can optimize $\theta$ with respect to $H$.
In the next section we show that if we can find a local optimum $\theta$ with respect to $H$, then we have found a local optimum
with respect to $J$.

\section{Theoretical Results}

\newtheorem{theorem}{Theorem}[section]
\newtheorem{lemma}{Lemma}[section]

We now show that Q-PAMDP(1) converges to a local or global optimum with mild assumptions.
We assume that iterating P-UPDATE converges to some $\theta^*$ with respect to a given objective function $f$.
As the P-UPDATE step is a design choice, it can be selected with the appropriate convergence property.
Q-PAMDP(1) is equivalent to the sequence 
\begin{align*}
\omega_{t+1} &= W(\theta_{t}) \\
\theta_{t+1} &= \text{P-UPDATE}(J_{\omega_{t+1}}, \theta_t), \\
\end{align*}
if Q-LEARN converges to $W(\theta)$ for each given $\theta$.

\begin{theorem}[Convergence to a Local Optimum]
For any $\theta_0$, if the sequence
\begin{align}
\label{eq:cond}
\theta_{t+1} = \text{P-UPDATE}(H, \theta_t),
\end{align}
converges to a local optimum with respect to $H$, then Q-PAMDP(1) converges to a local optimum with respect to $J$.
\end{theorem}

\begin{proof}
By definition of the sequence above $\omega_t = W(\theta_t)$, so it follows that
\begin{align*}
J_{\omega_t} = J(\theta, W(\theta)) = H(\theta).
\end{align*}
In other words, the objective function $J$ equals $H$ if $\omega = W(\theta)$.
Therefore, we can replace $J$ with $H$ in our update for $\theta$, to obtain the update rule
\begin{align*}
\theta_{t+1} &= \text{P-UPDATE}(H, \theta_t).
\end{align*}
Therefore by equation \ref{eq:cond} the sequence $\theta_t$
converges to a local optimum $\theta^*$ with respect to $H$.
Let $\omega^* = W(\theta^*)$.
As $\theta^*$ is a local optimum with respect to $H$, by definition there exists $\epsilon > 0$, $s.t.$
\begin{align*}
\left|\left|\theta^* - \theta\right|\right|_2 < \epsilon \implies H(\theta) \leq H(\theta^*).
\end{align*}
Therefore for any $\omega$,
\begin{align*}
\left|\left| \left( \begin{matrix} \theta^* \\ \omega^* \end{matrix} \right)
- \left( \begin{matrix} \theta \\ \omega \end{matrix} \right) \right|\right|_2 < \epsilon &\implies
\left|\left| \theta^* - \theta \right|\right|_2 < \epsilon \\
& \implies H(\theta) \leq H(\theta^*) \\
& \implies J(\theta, \omega) \leq J(\theta^*, \omega^*).
\end{align*}
Therefore $(\theta^*, \omega^*)$ is a local optimum with respect to $J$.
\end{proof}

In summary, if we can locally optimize $\theta$, and $\omega = W(\theta)$ at each step,
then we will find a local optimum for $J(\theta, \omega)$.
The conditions for the previous theorem can be met by 
assuming that P-UPDATE is a local optimization method such as a gradient based policy search.
A similar argument shows that if the sequence $\theta_t$ converges to a global optimum with respect to $H$,
then Q-PAMDP(1) converges to a global optimum $(\theta^*, \omega^*)$.

One problem is that at each step we must re-learn $W(\theta)$ for the updated value of $\theta$.
We now show that if updates to $\theta$ are bounded and $W(\theta)$ is a continuous function,
then the required updates to $\omega$ will also be bounded.
Intuitively, we are supposing that a small update to $\theta$ results in a small change
in the weights specifying which discrete action to choose.
The assumption that $W(\theta)$ is continuous is strong, and may not be satisfied by all PAMDPs.
It is not necessary for the operation of Q-PAMDP(1), but when it is satisfied we do not need to completely
re-learn $\omega$ after each update to $\theta$.
We show that by selecting an appropriate $\alpha$ we can shrink the differences in $\omega$ as desired.

\begin{theorem}[Bounded Updates to $\omega$]
If $W$ is continuous with respect to $\theta$, and updates to $\theta$ are of the form
\begin{align*}
\theta_{t+1} = \theta_t + \alpha_t \text{P-UPDATE}(\theta_t, \omega_t),
\end{align*}
with the norm of each P-UPDATE bounded by
\[0 < \left|\left|\text{P-UPDATE}(\theta_t, \omega_t)\right|\right|_2 < \delta,\]
for some $\delta > 0$,
then for any difference in $\omega$ $\epsilon >0$, there is an initial update rate $\alpha_0 > 0$ such that 
\[\alpha_t < \alpha_0 \implies \left|\left|\omega_{t+1} - \omega_t\right|\right|_2 < \epsilon.\]

\end{theorem}

\begin{proof}
Let $\epsilon > 0$ and
\[\alpha_0 = \frac{\delta}{\left|\left|\text{P-UPDATE}(\theta_t, \omega_t)\right|\right|_2}.\]
As $\alpha_t < \alpha_0,$ it follows that
\begin{align*}
\delta &> \alpha_t\left|\left|\text{P-UPDATE}(\theta_t, \omega_t)\right|\right|_2 \\
 &= \left|\left|\alpha_t\text{P-UPDATE}(\theta_t, \omega_t)\right|\right|_2 \\
 &= \left|\left|\theta_{t+1} - \theta_t\right|\right|_2.
\end{align*}
So we have
\[\left|\left|\theta_{t+1} - \theta_t\right|\right|_2 < \delta.\]
As $W$ is continuous, this means that
\[\left|\left|W(\theta_{t+1}) - W(\theta_t)\right|\right|_2 = \left|\left|\omega_{t+1} - \omega_t\right|\right|_2 < \epsilon.\]
\end{proof}

In other words, if our update to $\theta$ is bounded and $W$ is continuous, we can always adjust the learning rate $\alpha$
so that the difference between $\omega_{t}$ and $\omega_{t+1}$ is bounded.

With Q-PAMDP(1) we want P-UPDATE to optimize $H(\theta)$.
One logical choice would be to use a gradient update.
The next theorem shows that gradient of $H$ is equal to the gradient of $J$ if $\omega = W(\theta)$.
This is useful as we can apply existing gradient-based policy search methods to compute the gradient of $J$ with respect to $\theta$.
The proof follows from the fact that we are at a global optimum of $J$ with respect to $\omega$, and so the gradient
$\nabla_\omega J$ is zero.
This theorem requires that $W$ is differentiable (and therefore also continuous).

\begin{theorem}[Gradient of $H(\theta)$]
If $J(\theta, \omega)$ is differentiable with respect to $\theta$ and $\omega$ and $W(\theta)$ is differentiable with respect to $\theta$,
then the gradient of $H$ is given by $\nabla_\theta H(\theta) = \nabla_\theta J(\theta, \omega^*)$,
where $\omega^* = W(\theta)$.
\end{theorem}
\begin{proof}
If $\theta \in \mathbb{R}^n$ and $\omega \in \mathbb{R}^m$, then we can compute the gradient of $H$ by the chain rule:
\begin{align*}
\frac{\partial H(\theta)}{\partial \theta_i} &= \frac{\partial J(\theta, W(\theta))}{\partial \theta_i} \\
&= \sum\limits_{j=1}^n \frac{\partial J(\theta, \omega^*)}{\partial \theta_j} \frac{\partial \theta_j}{\partial \theta_i}
  +\sum\limits_{k=1}^m \frac{\partial J(\theta, \omega^*)}{\partial \omega^*_k} \frac{\partial \omega^*_k}{\partial \theta_i} \\
&=\frac{\partial J(\theta, \omega^*)}{\partial \theta_i}
  +\sum\limits_{k=1}^m \frac{\partial J(\theta, \omega^*)}{\partial \omega^*_k} \frac{\partial \omega_k^*}{\partial \theta_i},
\end{align*}
where $\omega^* = W(\theta)$.
Note that as by definitions of $W$,
\[\omega^* = W(\theta) = \arg\max\limits_\omega J(\theta, \omega),\]
we have that the gradient of $J$ with respect to $\omega$ is zero
$\partial J(\theta, \omega^*)/ \partial \omega^*_k = 0,$
as $\omega$ is a global maximum with respect to $J$ for fixed $\theta$.
Therefore, we have that
\[\nabla_\theta H(\theta) = \nabla_\theta J(\theta, \omega^*).\]
\end{proof}

To summarize, if $W(\theta)$ is continuous and P-UPDATE converges to a global or local optimum,
then Q-PAMDP(1) will converge to a global or local optimum, respectively,
and the Q-LEARN step will be bounded if the update rate of the P-UPDATE step is bounded.
As such, if P-UPDATE is a policy gradient update step then Q-PAMDP by Theorem 4.1 will converge to a local optimum and
by Theorem 4.4 the Q-LEARN step will require a fixed number of updates.
This policy gradient step can use the gradient of $J$ with respect to $\theta$.

With Q-PAMDP($\infty$) each step performs a full optimization on $\theta$ and then a full optimization of $\omega$.
The $\theta$ step would optimize $J(\theta, \omega)$, not $H(\theta)$, as we do update $\omega$ while we update $\theta$.
Q-PAMDP($\infty$) has the disadvantage of requiring global convergence properties for the P-UPDATE method.

\begin{theorem}[Local Convergence of Q-PAMDP($\infty$)]
If at each step of Q-PAMDP($\infty$) for some bounded set $\Theta$:
\begin{align*}
\theta_{t+1} &= \arg\max_{\theta \in \Theta} J(\theta, \omega_t), \\
\omega_{t+1} &= W(\theta_{t+1}),
\end{align*}
then Q-PAMDP($\infty$) converges to a local optimum.
\end{theorem}

\begin{proof}
By definition of $W$,
$\omega_{t+1} = \arg\max_{\omega} J(\theta_{t+1}, \omega).$
Therefore this algorithm takes the form of direct alternating optimization.
As such, it converges to a local optimum \cite{bezdek2002some}.
\end{proof}

Q-PAMDP($\infty$) has weaker convergence properties than Q-PAMDP(1), as it requires a globally convergent P-UPDATE.
However, it has the potential to bypass nearby local optima \cite{bezdek2002some}.

\section{Experiments}

We first consider a simplified robot soccer problem \cite{kitano1997robocup}
where a single striker attempts to score a goal against a keeper.
Each episode starts with the player at a random position along the bottom bound of the field.
The player starts with the ball in possession, and the keeper is positioned between the ball and the goal.
The game takes place in a 2D environment where the player and the keeper have a position, velocity and orientation and the ball has
a position and velocity, resulting in 14 continuous state variables.


An episode ends when the keeper possesses the ball, the player scores a goal, or the ball leaves the field.
The reward for an action is 0 for non-terminal state, 50 for a terminal goal state,
and $-d$ for a terminal non-goal state, where $d$ is the distance of the ball to the goal.
The player has two parameterized actions: kick-to$(x,y)$, which kicks to ball towards position $(x,y)$;
and shoot-goal($h$), which shoots the ball towards a position $h$ along the goal line.
Noise is added to each action.
If the player is not in possession of the ball, it moves towards it.
The keeper has a fixed policy: it moves towards the ball, and if the player shoots at the goal, the keeper moves to intercept the ball.

To score a goal, the player must shoot around the keeper.
This means that at some positions it must shoot left past the keeper, and at others to the right past the keeper.
However at no point should it shoot at the keeper, so an optimal policy is discontinuous.
We split the action into two parameterized actions: shoot-goal-left, shoot-goal-right.
This allows us to use a simple action selection policy instead of complex continuous action policy.
This policy would be difficulty to represent in a purely continuous action space, but is simpler in a parameterized action setting.

We represent the action-value function for the discrete action $a$ using linear function approximation
with Fourier basis features \cite{konidaris2011value}.
As we have 14 state variables, we must be selective in which basis functions to use.
We only use basis functions with two non-zero elements and exclude all velocity state variables.
We use the soft-max discrete action policy \cite{sutton1998introduction}.
We represent the action-parameter policy $\pi_\theta^{a}$ as a normal distribution around a weighted sum of features
$\pi_\theta^a(x | s) = \mathcal{N} (\theta_a^T \psi_a(s), \Sigma)$,
where $\theta_a$ is a matrix of weights, and $\psi_a(s)$ gives the features for state $s$, and $\Sigma$ is a fixed covariance matrix.
We use specialized features for each action.
For the shoot-goal actions we are using a simple linear basis $(1, g)$, where $g$ is the projection of the keeper onto the goal line.
For kick-to we use linear features $(1, bx, by, bx^2, by^2, (bx - kx)/\left|\left|b - k\right|\right|_2
,(by - ky)/\left|\left|b - k\right|\right|_2)$, where $(bx, by)$ is the position of the ball and $(kx, ky)$ is the position of the keeper.

For the direct policy search approach, we use the episodic natural actor critic (eNAC) algorithm \cite{peters2008natural},
computing the gradient of $J(\omega, \theta)$ with respect to $(\omega, \theta)$.
For the Q-PAMDP approach we use the gradient-descent Sarsa($\lambda$) algorithm for Q-learning,
and the eNAC algorithm for policy search.
At each step we perform one eNAC update based on 50 episodes and then refit $Q_\omega$ using 50 gradient descent Sarsa($\lambda$) episodes.

\begin{figure}[h]
    \includegraphics[scale=0.45]{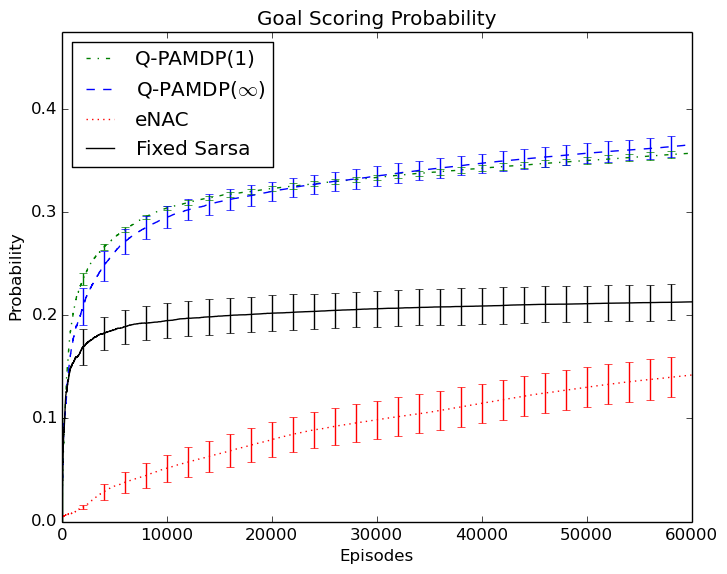}
    \centering
    \caption{
    Average goal scoring probability, averaged over 20 runs for Q-PAMDP(1), Q-PAMDP($\infty$), fixed parameter Sarsa, and eNAC in the goal domain.
    Intervals show standard error.}
    \label{fig:goals}
\end{figure}

Return is directly correlated with goal scoring probability, so their graphs are close to indentical.
As it is easier to interpret, we plot goal scoring probability in figure \ref{fig:goals}.
We can see that direct eNAC is outperformed by Q-PAMDP(1) and Q-PAMDP($\infty$).
This is likely due to the difficulty of optimizing the action selection parameters directly, rather than with Q-learning.

\begin{figure}[h]
    \includegraphics[scale=0.35]{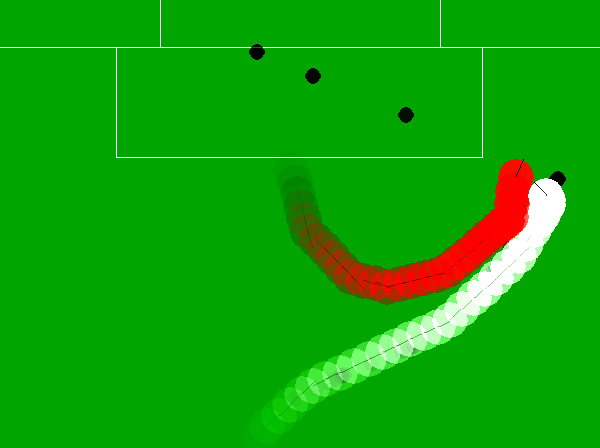}
    \centering
    \caption{A robot soccer goal episode using a converged Q-PAMDP(1) policy.
The player runs to one side, then shoots immediately upon overtaking the keeper.}
    \label{fig:episode}
\end{figure}

For both methods, the goal probability is greatly increased:
while the initial policy rarely scores a goal,
both Q-PAMDP(1) and Q-PAMDP($\infty$) increase the probability of a goal to roughly 35\%.
Direct eNAC converged to a local maxima of 15\%.
Finally, we include the performance of SARSA($\lambda$) where the action parameters are fixed at the initial $\theta_0$.
This achieves roughly 20\% scoring probability.
Both Q-PAMDP(1) and Q-PAMDP($\infty$) strongly out-perform fixed parameter SARSA, but eNAC does not.
Figure \ref{fig:episode} depicts a single episode using a converged Q-PAMDP(1) policy---
the player draws the keeper out and strikes when the goal is open.

\begin{figure}[t]
    \begin{subfigure}[t]{0.129\textwidth}
        \includegraphics[width=\textwidth]{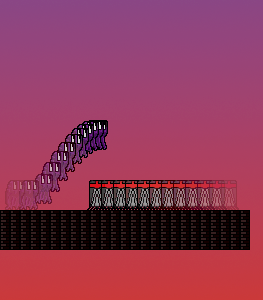}
    \end{subfigure}
    ~
    ~
    \begin{subfigure}[t]{0.135\textwidth}
        \includegraphics[width=\textwidth]{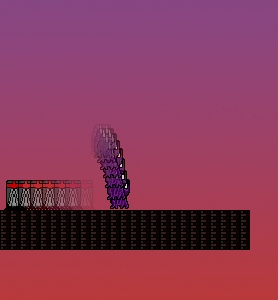}
    \end{subfigure}
    ~
    ~
    \begin{subfigure}[t]{0.132\textwidth}
        \includegraphics[width=\textwidth]{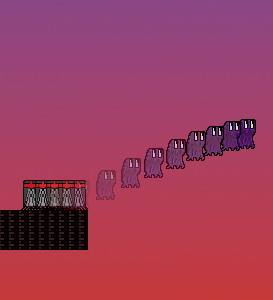}
    \end{subfigure}
    \centering
\caption{A screenshot from the Platform domain. The player hops over an enemy, and then leaps over a gap.}
\label{fig:plats}
\end{figure}

Next we consider the Platform domain, where the agent starts on a platform and must reach a goal while avoiding enemies. 
If the agent reaches the goal platform, touches an enemy, or falls into a gap between platforms, the episode ends.
This domain is depicted in figure \ref{fig:plats}.
The reward for a step is the change in $x$ value for that step, divided by the total length of all the platforms and gaps.
The agent has two primitive actions: run or jump, which continue for a fixed period or until the agent lands again respectively.
There are two different kinds of jumps: a high jump to get over enemies, and a long jump to get over gaps between platforms.
The domain therefore has three parameterized actions: run($dx$), hop($dx$), and leap($dx$).
The agent only takes actions while on the ground, and enemies only move when the agent is on their platform.
The state space consists of four variables $(x, \dot{x}, ex, \dot{ex})$,
representing the agent position, agent speed, enemy position, and enemy speed respectively.
For learning $Q_\omega$, as in the previous domain, we use linear function approximation with the Fourier basis.
We apply a softmax discrete action policy based on $Q_\omega$, and a Gaussian parameter policy based on scaled parameter features $\psi_a(s)$.

\begin{figure}[h]
    \includegraphics[scale=0.45]{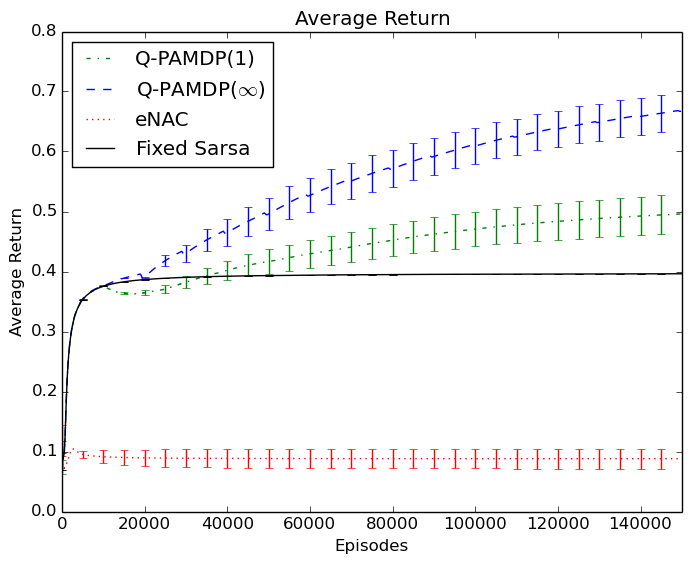}
    \centering
    \caption{
    Average percentage distance covered, averaged over 20 runs for Q-PAMDP(1), Q-PAMDP($\infty$), and eNAC in the Platform domain.
    Intervals show standard error.}
    \label{fig:return}
\end{figure}

\begin{figure}[h]
    \includegraphics[scale=0.225]{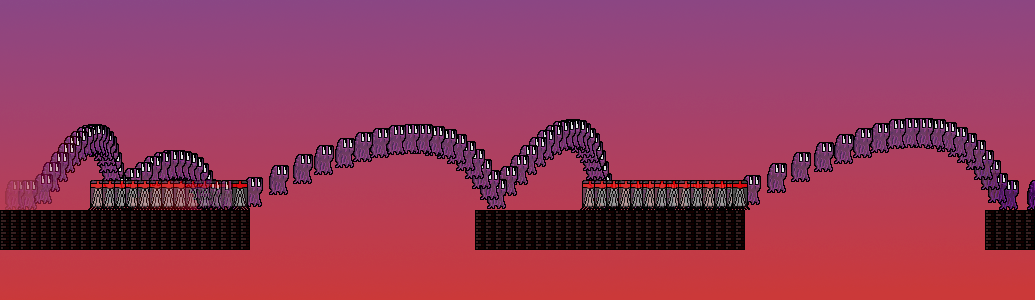}
    \centering
    \caption{A successful episode of the Platform domain.
The agent hops over the enemies, leaps over the gaps, and reaches the last platform.}
    \label{fig:plat_episode}
\end{figure}

Figure \ref{fig:return} shows the performance of eNAC, Q-PAMDP(1), Q-PAMDP($\infty$), and SARSA with fixed parameters.
Both Q-PAMDP(1) and Q-PAMDP($\infty$) outperformed the fixed parameter SARSA method,
reaching on average 50\% and 65\% of the total distance respectively.
We suggest that Q-PAMDP($\infty$) outperforms Q-PAMDP(1) due to the nature of the Platform domain.
Q-PAMDP(1) is best suited to domains with smooth changes in the action-value function with respect to changes in the parameter-policy.
With the Platform domain, our initial policy is unable to make the first jump without modification.
When the policy can reach the second platform, we need to drastically change the action-value function to account for this platform.
Therefore, Q-PAMDP(1) may be poorly suited to this domain as the small change in parameters
that occurs between failing to making the jump and actually making it results in a large change in the action-value function.
This is better than the fixed SARSA baseline of 40\%, and much better than direct optimization using eNAC which reached 10\%.
Figure \ref{fig:plat_episode} shows a successfully completed episode of the Platform domain.

\section{Related Work}

Hauskrecht \textit{et al.} \shortcite{hauskrecht2004factored} introduced an algorithm for solving factored MDPs with a hybrid discrete-continuous action space.
However, their formalism has an action space with a mixed set of discrete and continuous components, whereas our domain has distinct actions
with a different number of continuous components for each action.
Furthermore, they assume the domain has a compact factored representation, and only consider planning.

Rachelson \shortcite{rachelsontemporal} encountered parameterized actions in the form of an action to wait for
a given period of time in his research on time dependent, continuous time MDPs (TMDPs).
He developed XMDPs, which are TMDPs with a parameterized action space \cite{rachelsontemporal}.
He developed a Bellman operator for this domain, and in a later paper mentions that the TiMDP$_{poly}$ algorithm can work with parameterized actions,
although this specifically refers to the parameterized wait action \cite{rachelson2009timdppoly}.
This research also takes a planning perspective, and only considers a time dependent domain.
Additionally, the size of the parameter space for the parameterized actions is the same for all actions.

Hoey \textit{et al.} \shortcite{hoey2013bayesian} considered mixed discrete-continuous actions in their work on Bayesian affect control theory.
To approach this problem they use a form of POMCP, a Monte Carlo sampling algorithm, using domain specific adjustments to
compute the continuous action components \cite{silver2010monte}.
They note that the discrete and continuous components of the action space reflect different control aspects:
the discrete control provides the ``what'', while the continuous control describes the ``how'' \cite{hoey2013bayesian}.

In their research on symbolic dynamic programming (SDP) algorithms, Zamani \textit{et al.} \shortcite{zamani2012symbolic}
considered domains with a set of discrete parameterized actions.
Each of these actions has a different parameter space.
Symbolic dynamic programming is a form of planning for relational or first-order MDPs,
where the MDP has a set of logical relationships defining its dynamics and reward function.
Their algorithms represent the value function as an extended algebraic decision diagram (XADD),
and is limited to MDPs with predefined logical relations.

A hierarchical MDP is an MDP where each action has subtasks.
A subtask is itself an MDP with its own states and actions which may have their own subtasks.
Hierarchical MDPs are well-suited for representing parameterized actions as we could consider
selecting the parameters for a discrete action as a subtask.
MAXQ is a method for value function decomposition of hierarchical MDPs \cite{dietterich2000hierarchical}.
One possiblity is to use MAXQ for learning the action-values in a parameterized action problem.

\section{Conclusion}

The PAMDP formalism models reinforcement learning domains with parameterized actions.
Parameterized actions give us the adaptibility of continuous domains and to use distinct kinds of actions.
They also allow for simple representation of discontinuous policies without complex parameterizations.
We have presented three approaches for model-free learning in PAMDPs: direct optimization and two variants of the Q-PAMDP algorithm.
We have shown that Q-PAMDP(1), with an appropriate P-UPDATE method, converges to a local or global optimum.
Q-PAMDP($\infty$) with a global optimization step converges to a local optimum.

We have examined the performance of these approaches in the goal scoring domain and the Platformer domain.
The robot soccer goal domain models the situation where a striker must out-maneuver a keeper to score a goal.
Of these, Q-PAMDP(1) and Q-PAMDP($\infty$) outperformed eNAC and fixed parameter SARSA.
Q-PAMDP(1) and Q-PAMDP($\infty$) performed similarly well in terms of goal scoring,
learning policies that score goals roughly 35\% of the time.
In the Platform domain we found that both Q-PAMDP(1) and Q-PAMDP($\infty$) outperformed eNAC and fixed SARSA.
\bibliography{aaai}
\bibliographystyle{aaai}

\end{document}